\documentclass{article}

\PassOptionsToPackage{numbers, compress}{natbib}

\usepackage[preprint]{neurips_2026}

\usepackage[utf8]{inputenc} 
\usepackage[T1]{fontenc}    
\usepackage{hyperref}       
\usepackage{url}            
\usepackage{booktabs}       
\usepackage{amsfonts}       
\usepackage{nicefrac}       
\usepackage{microtype}      
\usepackage{xcolor}         
\usepackage{amsfonts}       
\usepackage{nicefrac}       
\usepackage{microtype}      
\usepackage{amsmath,amssymb,amsthm}
\usepackage[ruled,vlined]{algorithm2e}
\usepackage{verbatim}
\usepackage{braket}
\usepackage{todonotes}
\usepackage{paralist}
\usepackage{caption}
\usepackage{wrapfig}
\usepackage{graphicx}
\usepackage{caption}
\usepackage{multicol}
\usepackage{multirow}
\usepackage{subfigure}
\usepackage{float}
\usepackage{textcomp,booktabs}
\usepackage{color}
\usepackage{subcaption}
\usepackage{enumitem} 
\usepackage[table]{xcolor}

\definecolor{mygray}{gray}{.9}

\newcommand{\ourmodel}{{S$^2$Aligner}\xspace}
\newcommand{\nop}[1]{}

\newtheorem{theorem}{Theorem}[section]

\makeatletter
  \newcommand\figcaption{\def\@captype{figure}\caption}
  \newcommand\tabcaption{\def\@captype{table}\caption}
\makeatother
\title{\ourmodel: Pair-Efficient Transferable Pre-Training \\ for Sparse Text-Attributed Graphs}

\author{%
Yuhan Wang\textsuperscript{1},
Haopeng Zhang\textsuperscript{1},
Yibo Ding\textsuperscript{1},
Jiaqi Yu\textsuperscript{1},
Xinyu Zhao\textsuperscript{1} \\
\textbf{Yuhang Liu\textsuperscript{2},
Ziwei Zhang\textsuperscript{1},
Xiao Wang\textsuperscript{1},
Ruijie Wang\textsuperscript{1}\thanks{Corresponding author: \texttt{ruijiew@buaa.edu.cn}.}} \\
\textsuperscript{1}Beihang University, Beijing, China\\
\textsuperscript{2}Tianjin University, Tianjin, China\\
}

\begin{document}
\maketitle
\begin{abstract}

Pre-training on text-attributed graphs (TAGs) is central to building transferable graph foundation models, where LLM-as-Aligner methods align graph and text representations through the semantic knowledge of large language models. However, these methods usually assume that node texts provide sufficient and reliable supervision, an assumption often violated in real-world sparse TAGs. When textual anchors are missing, noisy, or uneven across domains, graph structures must be aligned with weak semantic evidence, leading to unreliable structure--semantics correspondence and sparsity-induced transfer bias. This paper presents \textbf{\ourmodel}, a \textbf{S}parsity-aware and \textbf{S}tructure-enhanced LLM-as-Aligner framework for graph--text pre-training on sparse TAGs. The key idea is to decouple semantic alignment from structural modeling, allowing topology-aware signals to enhance alignment without contaminating the shared semantic space. Specifically, \ourmodel decomposes graph--text representations into semantic and structural components, uses structure-oriented reconstruction with consistency control to inject reliable topology cues into text representations, and suppresses inconsistent structural signals under textual sparsity. Moreover, \ourmodel introduces sparsity-aware cross-domain risk balancing, which calibrates domain risks through a global-domain density ratio and downweights unreliable sparse samples via graph reliability estimation. Theoretical analysis shows that this objective reduces cross-domain generalization gaps by controlling domain risk discrepancy. Extensive experiments across diverse graph domains, sparsity levels, and downstream tasks demonstrate that \ourmodel consistently outperforms existing baselines. 


\end{abstract}

\section{Introduction}
\vspace{-1mm}
\label{introduction}
Graph foundation models (GFMs) have recently emerged as a promising paradigm for inductive graph learning. They aim to distill transferable structural and semantic priors from large-scale graph corpora, reducing the need for task-specific supervision and costly fine-tuning~\cite{DBLP:journals/corr/abs-2310-11829,huang2024can,ren2024survey}. This paradigm is particularly suitable for text-attributed graphs (TAGs), common in citation networks, social platforms, and e-commerce systems~\cite{paper1,feng2024taglas,chen2022brainnet,chen2020multi}, where node texts provide rich semantics and edges capture relational structure. Recently, cross-modal graph-text alignment has become a central direction for TAG pretraining~\cite{G2P2,zhu2025graphclip,adaligner}, aligning graph or subgraph embeddings with textual embeddings, often leveraging large language models (LLMs) as aligners — the so-called LLM-as-Aligner paradigm (Figure~\ref{fig:problem_definition}).

Despite its promise, the LLM-as-Aligner paradigm relies on a strong assumption: textual evidence should be sufficient to provide reliable supervision for graph-text alignment~\cite{G2P2,zhu2025graphclip,adaligner}. This assumption is often violated in real-world TAGs. Node texts can be missing, short, noisy, or highly uneven across domains. In this case, the text view becomes an incomplete anchor for the graph view. Rich graph structures must be aligned with weak and ambiguous semantic signals, making the learned structure--semantics correspondence unreliable. The problem is further amplified by recent methods that use LLM-generated subgraph summaries as supervision~\cite{zhu2025graphclip,adaligner}. When textual contexts are sparse, LLMs have to infer missing semantics from incomplete evidence, which may introduce uncertainty and bias into the alignment signal. This raises a fundamental question:

{\centering \textit{
how can graph-text pretraining learn transferable structure--semantics correspondence when the cross-modal supervision itself is sparse and unreliable?
}}

This paper studies LLM-as-Aligner pretraining on sparse TAGs, preserving transferable pretraining while reducing dependence on dense textual supervision. It introduces two coupled challenges:

\textbf{Challenge 1: Unreliable structure--semantics alignment. }
In sparse text-attributed graphs, textual anchors are incomplete or ambiguous, while the graph structure remains rich. This imbalance can cause the model to misalign structural patterns with noisy text, failing to learn faithful structure--semantics correspondences. Figure~\ref{fig:motivation} illustrates this issue: (a) text sparsity increases semantic uncertainty, weakening alignment anchors; (b) structural supplementation helps under dense text but may amplify mismatches under sparse text; (c) sparse-text graph--text retrieval performs substantially worse than full-text retrieval, highlighting the limited robustness of LLM-as-Aligner models.

\textbf{Challenge 2: Sparsity-induced transfer bias.}
Sparse texts also weaken zero-shot generalization. Different domains require different densities of alignment signals, since some domains can transfer with coarse semantic anchors while others depend on finer graph--text correspondences. Text sparsity makes such reliability uneven across domains. Standard pretraining may therefore overfit to domains with sufficient anchors or learn biased correlations from noisy ones, causing imbalanced domain risks and unstable transfer to unseen sparse TAGs. Thus, robust cross-domain generalization requires sparsity-aware risk balancing to extract domain-invariant representations under sparse supervision.

\begin{figure}[t]
    \centering

    \begin{minipage}[t]{0.22\textwidth}
        \centering
        \includegraphics[width=\linewidth]{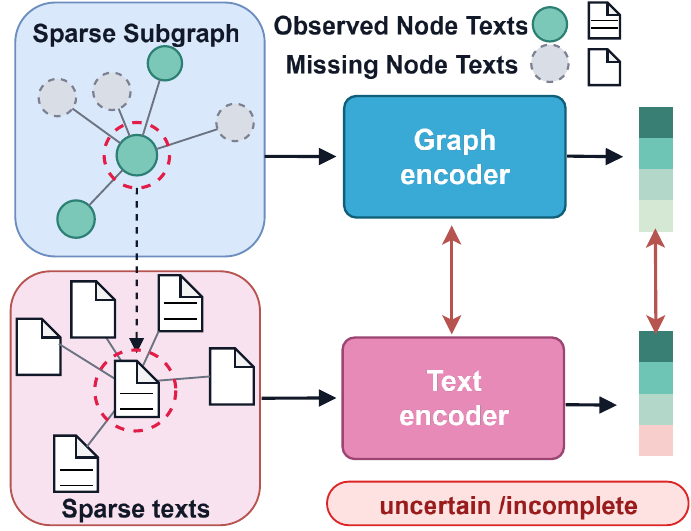}
        \captionof{figure}{
        LLM-as-Aligner.
        }
        \label{fig:problem_definition}
    \end{minipage}
    \hfill
    \begin{minipage}[t]{0.77\textwidth}
        \centering
        \begin{minipage}[t]{0.33\linewidth}
            \centering
            \includegraphics[width=\linewidth]{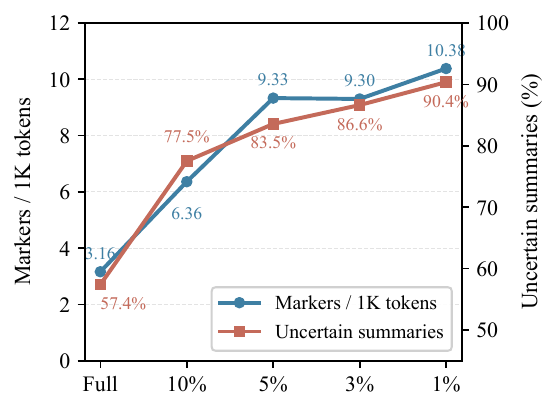}
            \vspace{-2mm}
            \centerline{\footnotesize (a) Uncertainty vs. Sparsity}
        \end{minipage}
        \hfill
        \begin{minipage}[t]{0.32\linewidth}
            \centering
            \includegraphics[width=\linewidth]{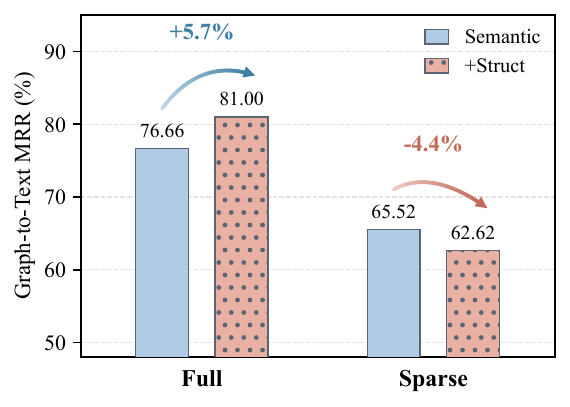}
            \vspace{-2mm}
            \centerline{\footnotesize (b) Structural supplementation}
        \end{minipage}
        \hfill
        \begin{minipage}[t]{0.33\linewidth}
            \centering
            \includegraphics[width=\linewidth]{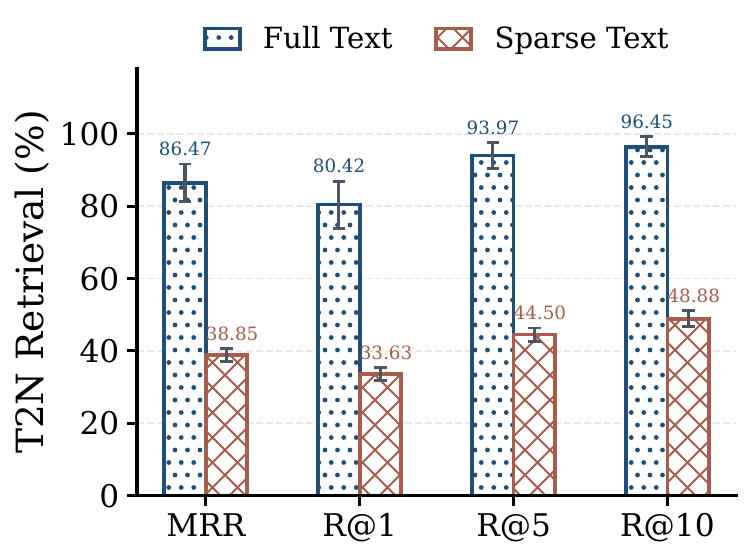}
            \vspace{-2mm}
            \centerline{\footnotesize (c) Full vs. sparse retrieval}
        \end{minipage}

        \captionof{figure}{
        (a) Summary uncertainty increases.
        (b) Naive structural supplementation becomes harmful. 
        (c) Full vs. Sparse alignment.}

        \label{fig:motivation}
    \end{minipage}
\end{figure}

To address the above challenges, we propose \textbf{\ourmodel}, a sparsity-aware and structure-enhanced LLM-as-Aligner framework for graph--text pretraining on sparse TAGs. The overall idea is to decouple semantic alignment from structural modeling, so that sparse and uncertain structural cues can be used to enhance graph--text alignment without directly contaminating the shared semantic space. For \textbf{Challenge 1}, \ourmodel separates each graph--text representation into semantic and structural components. The semantic components define the main alignment space, while the structural components are modeled through structure-oriented reconstruction. By using reconstruction consistency to control structural enhancement, \ourmodel injects reliable topology-aware signals into text representations and suppresses inconsistent structural cues, thereby mitigating structural negative transfer under textual sparsity. For \textbf{Challenge 2}, \ourmodel introduces a sparsity-aware cross-domain risk balancing mechanism. It calibrates domain risks with a global-domain density ratio and further uses graph reliability to reduce the impact of unreliable sparse samples. This encourages the model to focus on domain-invariant graph--text correspondences rather than sparse domain-specific correlations. Theoretical analysis suggests that this objective can reduce the cross-domain generalization gap by controlling domain risk discrepancy.

Our contributions are summarized as follows:

\begin{list}{\labelitemi}{\leftmargin=1em \itemindent=0em}

\item We propose \ourmodel, a sparsity-aware and structure-enhanced LLM-as-Aligner framework for sparse TAGs. It transforms uncertain topological signals into verifiable structural enhancements and incorporates sparsity-aware risk balancing for better cross-domain generalization.

\item We design a sparsity-aware cross-domain risk balancing mechanism. It filters noisy samples using graph reliability evaluation and calibrates cross-domain risks via a global-domain density ratio, reducing domain-biased alignment and improving generalization.

\item Extensive experiments validate our method across sparsity levels, graph domains, and downstream tasks. Notably, even with only $10\%$ textual attributes, it surpasses baselines trained with full attributes, demonstrating superior sparsity generalization and data efficiency.

\end{list}

\section{Preliminary}
\label{sec:preliminary}
\vspace{-1mm}

\paragraph{LLM-as-Aligner for graph problems.}
Following GraphCLIP~\cite{zhu2025graphclip}, we define a text-attributed graph (TAG) as
$\mathcal{G}=(\mathcal{V},\mathcal{E},\mathcal{T})$, where $\mathcal{V}$ and $\mathcal{E}$ denote the node and edge sets, and
$\mathcal{T}=\{(g_i,t_i)\}_{i=1}^{M}$ denotes a collection of graph-text pairs.
Here, $g_i$ can be a node, a sampled subgraph, or an entire graph, and $t_i$ is the associated textual information, such as a textual attribute, structural summary, or label prompt. Given a graph encoder $f_{\mathrm{G}}$ and a text encoder $f_{\mathrm{T}}$, each pair is encoded as 
\begin{equation} 
\mathbf{z}_{g,i} = \operatorname{READOUT}\!\left(f_{\mathrm{G}}(g_i)\right), \qquad 
\mathbf{z}_{t,i} = \operatorname{READOUT}\!\left(f_{\mathrm{T}}(t_i)\right), 
\end{equation}

For a mini-batch $\mathcal{T}$, graph and text representations are aligned by a symmetric contrastive loss:
\begin{equation}
\mathcal{L}_{\mathrm{align}}(\mathcal{T})
=
\frac{1}{2|\mathcal{T}|}
\sum_{(g_i,t_i) \in\mathcal{T}}
\left[
-
\log
\frac{
\exp\left(\operatorname{sim}(\mathbf{z}_{g,i},\mathbf{z}_{t,i})/\tau\right)
}{
\sum_{j}
\exp\left(\operatorname{sim}(\mathbf{z}_{g,i},\mathbf{z}_{t,j})/\tau\right)
}
-
\log
\frac{
\exp\left(\operatorname{sim}(\mathbf{z}_{g,i},\mathbf{z}_{t,i})/\tau\right)
}{
\sum_{j}
\exp\left(\operatorname{sim}(\mathbf{z}_{g,j},\mathbf{z}_{t,i})/\tau\right)
}
\right],
\label{eq:pre_align}
\end{equation}
where $\operatorname{sim}(\cdot,\cdot)$ denotes cosine similarity and $\tau$ is a temperature parameter. 
This objective treats the matched graph--text pair as the positive pair and all other pairs in the mini-batch as negatives.

\paragraph{Sparse text-attributed graph problem.}
In real-world TAGs, textual attributes are often missing, incomplete, or unreliable.
We define a sparse TAG as
$\widetilde{\mathcal{G}}=(\mathcal{V},\mathcal{E},\widetilde{\mathcal{T}})$, where
$\widetilde{\mathcal{T}}\subseteq\mathcal{T}$ denotes the observed sparse graph--text pair set.
A straightforward solution is to optimize the graph and text encoders over the sparse pairs:
\begin{equation}
    \min_{\theta_{\mathrm{G}},\theta_{\mathrm{T}}}
    \mathcal{L}_{\mathrm{align}}(\widetilde{\mathcal{T}}),
\end{equation}
where $\theta_{\mathrm{G}}$ and $\theta_{\mathrm{T}}$ are the parameters of $f_{\mathrm{G}}$ and $f_{\mathrm{T}}$, respectively.
However, directly aligning sparse graph--text pairs can be unstable, since missing or unreliable textual attributes provide weak supervision and may introduce noisy cross-modal correspondences.

Although the above formulation is general, we instantiate $g_i$ as a sampled subgraph in practice to exploit local structural context.
After pre-training, zero-shot prediction is performed by converting each candidate label $c\in\mathcal{Y}$ into a textual description $t_c$ and selecting the label with the highest graph--prompt similarity:
\begin{equation}
    \hat{y}
    =
    \arg\max_{c\in\mathcal{Y}}
    \operatorname{sim}
    \left(
    f_{\mathrm{G}}(g),
    f_{\mathrm{T}}(t_c)
    \right).
\end{equation}



\section{Method}
\vspace{-2mm}
\label{Method}
\begin{figure}[t]
    \centering
    \includegraphics[width=\linewidth]{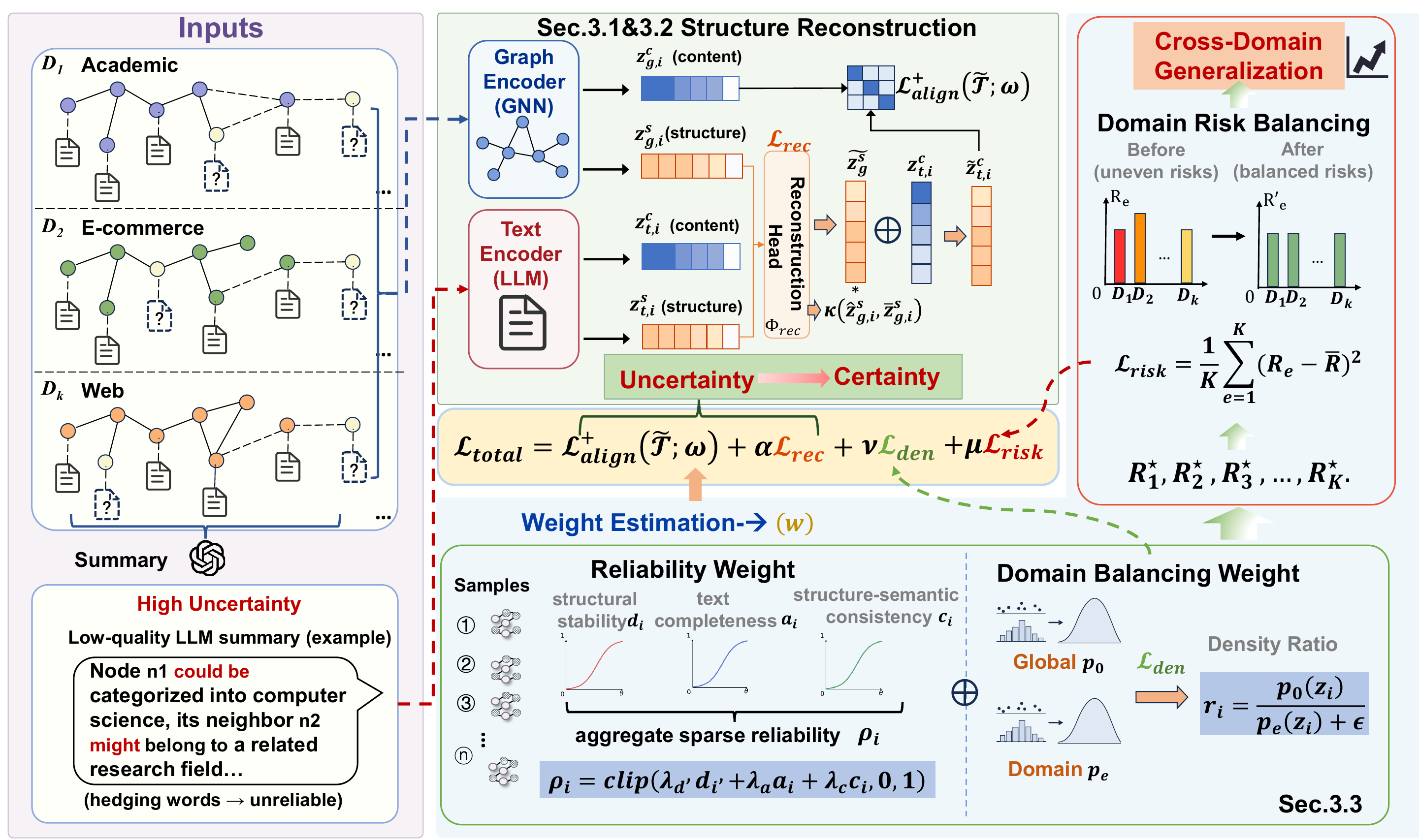}
    \caption{The overall framework of \ourmodel is shown in the figure above.  
It encodes sparse text-attributed graphs into content and structural components and applies latent reconstruction on the structural branch to reduce negative transfer from sparse text. We further introduce Sparse-aware Cross-domain Risk Balancing, aligning multi-source domain risks via density estimation and reliability weighting to learn domain-invariant features for robust cross-domain generalization.}
    \label{fig:framework}
\end{figure}
In this section, we present \textbf{\ourmodel}, a sparsity- and structure-enhanced LLM-as-Aligner framework (Figure~\ref{fig:framework}) for pretraining on sparse text-attributed graphs. 
First, we decouple graph--text representations into semantic and structural components, separating semantic alignment from structural modeling (Sec.~\ref{sec:disentangle}). 
Next, structure-oriented reconstruction injects reliable structural signals while suppressing inconsistent cues (Sec.~\ref{sec:reconstruction}). 
Finally, Sparse-aware Cross-domain Risk Balancing reweights source-domain risks using cross-domain support and graph reliability to enhance zero-shot cross-graph generalization (Sec.~\ref{sec:scrb}). 
Together, these components enable robust graph--text alignment under sparse supervision and enhance transferability to unseen domains.

\subsection{Content-Structure Factorization}
\vspace{-1mm}
\label{sec:disentangle}


In text-sparse cross-modal alignment scenarios, topological structure is crucial for capturing graph–text correspondences. 
However, structural information generated from text often exhibits high uncertainty, which may introduce noise and lead to negative transfer. 
To address this issue, we disentangle content semantics and structural information into two independent branches:
\vspace{-2mm}
\begin{itemize}[leftmargin = 8pt]
    \item \textbf{Semantic alignment branch:} relies solely on deterministic content information, extracting stable semantics from node attributes while avoiding interference from high-uncertainty structural text.
    \item \textbf{Structural branch:} kept independent and purified through a structure reconstruction task, extracting reliable topological information to fully leverage the positive transfer effect of structure.
\end{itemize}
\vspace{-2mm}
For each graph sample $g$, we construct two complementary text views: a semantic summary $S$ that characterizes attribute semantics, and a structure-aware description $S^s$ that represents topological features. The semantic summary focuses on the content information at the node attribute level, while the structure-aware description captures neighborhood patterns, connection topology, and the structural roles of nodes.

On the graph modality side, we learn and obtain the graph content embedding $\mathbf{z}_{g,i}^c$ and graph structure embedding $\mathbf{z}_{g,i}^s$ from the subgraph representation and topological structure, respectively. On the text modality side, the two types of text views are encoded into the text content embedding $\mathbf{z}_{t,i}^c$ and text structure embedding $\mathbf{z}_{t,i}^s$, respectively.

The content embedding pair $(\mathbf{z}_{g,i}^c, \mathbf{z}_{t,i}^c)$ forms the main semantic alignment space; the structural embedding pair $(\mathbf{z}_{g,i}^s, \mathbf{z}_{t,i}^s)$ is used for the topology-oriented structure reconstruction in Sec.~\ref{sec:reconstruction}, which completes the denoising and purification of structural text representations as well as the modeling of positive transfer.



\subsection{Structure-Oriented Reconstruction}
\vspace{-1mm}
\label{sec:reconstruction}
To filter out noise and redundancy in high-uncertainty structural text while extracting reliable topological information, 
we use the graph-side structural embedding as a gradient-detached topological target and train the text-side structural embedding to reconstruct it:
\begin{equation}
    \mathcal{L}_{\mathrm{rec}}
    =
    \frac{1}{|\widetilde{\mathcal{T}}|}
    \sum_{(g_i,t_i)\in\widetilde{\mathcal{T}}}
    \left[
    1-
    \operatorname{sim}
    \left(
    \widehat{\mathbf{z}}_{g,i}^s,
    \bar{\mathbf{z}}_{g,i}^s
    \right)
    \right]
    =
    \frac{1}{|\widetilde{\mathcal{T}}|}
    \sum_{(g_i,t_i)\in\widetilde{\mathcal{T}}}
    \left[
    1-
    \operatorname{sim}
    \left(
    \phi_{\mathrm{rec}}(\mathbf{z}_{t,i}^s),
    \operatorname{detach}(\mathbf{z}_{g,i}^s)
    \right)
    \right],
    \label{eq:rec}
\end{equation}
where $\widehat{\mathbf{z}}_{g,i}^s=\phi_{\mathrm{rec}}(\mathbf{z}_{t,i}^s)$ is the reconstructed structural embedding, $\bar{\mathbf{z}}_{g,i}^s=\operatorname{detach}(\mathbf{z}_{g,i}^s)$ is the gradient-detached topology target, and $\phi_{\mathrm{rec}}(\cdot)$ is a lightweight reconstruction head.

We further use the reconstruction consistency as a reliability gate for structural injection.
Instead of introducing a separate score equation, we directly define the enhanced text embedding as
\begin{equation}
    \widetilde{\mathbf{z}}_{t,i}^c
    =
    \operatorname{norm}
    \left(
    \mathbf{z}_{t,i}^c
    +
    \kappa
    \left(
    \widehat{\mathbf{z}}_{g,i}^s,
    \bar{\mathbf{z}}_{g,i}^s
    \right)
    \widehat{\mathbf{z}}_{g,i}^s
    \right),
    \qquad
    \kappa(\mathbf{a},\mathbf{b})
    =
    \frac{
    1+
    \operatorname{sim}(\mathbf{a},\mathbf{b})
    }{2}.
    \label{eq:enhanced_text}
\end{equation}
Here, $\kappa(\cdot,\cdot)\in[0,1]$ measures how well the structure-aware text recovers graph topology.
If the structural description is unreliable, the gate shrinks and the injected structural signal is suppressed.
Thus, only recoverable, topology-consistent structural information is injected into the text anchor.

For each sample, we define the enhanced contrastive loss as
\begin{equation}
\ell_i^{+}
=
\frac{1}{2}
\left[
-
\log
\frac{
\exp\left(
\operatorname{sim}
\left(
\mathbf{z}_{g,i}^c,
\widetilde{\mathbf{z}}_{t,i}^c
\right)/\tau
\right)
}{
\sum_{(g_j,t_j)\in\widetilde{\mathcal{T}}}
\exp\left(
\operatorname{sim}
\left(
\mathbf{z}_{g,i}^c,
\widetilde{\mathbf{z}}_{t,j}^c
\right)/\tau
\right)
}
-
\log
\frac{
\exp\left(
\operatorname{sim}
\left(
\mathbf{z}_{g,i}^c,
\widetilde{\mathbf{z}}_{t,i}^c
\right)/\tau
\right)
}{
\sum_{(g_j,t_j)\in\widetilde{\mathcal{T}}}
\exp\left(
\operatorname{sim}
\left(
\mathbf{z}_{g,j}^c,
\widetilde{\mathbf{z}}_{t,i}^c
\right)/\tau
\right)
}
\right].
\label{eq:sample_align}
\end{equation}
To connect the reconstruction module with the risk balancing module, we write the batch-level alignment loss in a weighted form. Therefore, the alignment loss and base training loss is:
\begin{equation}
    \mathcal{L}_{\mathrm{align}}^{+}
    \left(
    \widetilde{\mathcal{T}};
    \boldsymbol{\omega}
    \right)
    =
    \frac{
    \sum_{(g_i,t_i)\in\widetilde{\mathcal{T}}}
    \omega_i
    \ell_i^{+}
    }{
    \sum_{(g_i,t_i)\in\widetilde{\mathcal{T}}}
    \omega_i
    },
    ~~~
    \mathcal{L}_{\mathrm{base}}
    =
    \mathcal{L}_{\mathrm{align}}^{+}
    \left(
    \widetilde{\mathcal{T}};
    \boldsymbol{\omega}
    \right)
    +
    \alpha
    \mathcal{L}_{\mathrm{rec}},
    \label{eq:base_loss}
\end{equation}
where $\alpha$ controls the strength of structure-oriented reconstruction. Next, we introduce how to design $\boldsymbol{\omega}$ to adaptively adjust each sample's contribution to the overall loss and to guide the model to focus on invariant features across modalities and views.


\subsection{Sparse-aware Cross-domain Risk Balancing }
\label{sec:scrb}
\vspace{-1mm}
Sparse graph--text pre-training usually involves multiple source domains with different graph structures, textual sparsity levels, and noise patterns.
Let $\{\mathcal{D}_e\}_{e=1}^{K}$ denote $K$ source domains, where each domain contains graph--text pairs $(g_i,t_i)$.
Although the enhanced alignment loss in Eq.~\eqref{eq:base_loss} improves sample-level robustness, directly minimizing the pooled empirical risk can still overfit domain-specific sparse patterns.
To address this issue, we propose \textbf{Sparse-aware Cross-domain Risk Balancing}, which assigns each sample a reliability-aware density-ratio weight and further regularizes the risks across source domains.

\textbf{Shared-support density ratio.}
For each pair $(g_i,t_i)$ from domain $\mathcal{D}_{e(i)}$, let $\mathbf{z}_i$ denote representation of each sample. We estimate two densities over $\mathbf{z}_i$: a global density $p_{0}(\mathbf{z}_i)$ over all source domains and a domain-specific density $p_{e(i)}(\mathbf{z}_i)$ within its source domain.
The density ratio is defined as
\begin{equation}
    r_i
    =
    \left(
    \frac{
    p_{0}(\mathbf{z}_i)
    }{
    p_{e(i)}(\mathbf{z}_i)+\epsilon
    }
    \right)^{\gamma},
    \qquad
    \mathcal{L}_{\mathrm{den}}
    =
    -
    \frac{1}{|\widetilde{\mathcal{T}}|}
    \sum_{(g_i,t_i)\in\widetilde{\mathcal{T}}}
    \left[
    \log p_{0}(\mathbf{z}_i)
    +
    \log p_{e(i)}(\mathbf{z}_i)
    \right],
    \label{eq:density_ratio}
\end{equation}
where $\epsilon$ is a small constant for numerical stability, $\gamma$ controls the sharpness of density-ratio reweighting, and $\widetilde{\mathcal{T}}$ denotes the current mini-batch.
Intuitively, $r_i$ upweights samples lying in the cross-domain shared support and downweights samples dominated by domain-specific structural bias.

\textbf{Sparse graph reliability.}
Density compatibility alone cannot identify unreliable sparse samples.
We therefore introduce a sparse reliability score $\rho_i$ to measure whether the graph--text pair provides trustworthy supervision.
Specifically, we consider three complementary signals: structural stability, textual completeness, and structure--semantic consistency:
\begin{equation}
    d_i
    =
    \exp\left(\frac{1}{\log N_i}\sum_{j\in\mathcal{V}(g_i)}p_{i,j}\log p_{i,j}\right),
    \quad
    a_i
    =
    \frac{
    |\mathcal{V}_{t}(g_i)|
    }{
    |\mathcal{V}(g_i)|
    },
    \quad
    c_i
    =
    \frac{
    1+
    \operatorname{sim}
    \left(
    \mathbf{z}_{g,i}^{c},
    \mathbf{z}_{t,i}^{c}
    \right)
    }{2}.
    \label{eq:reliability_components}
\end{equation}
Here, $p_{i,j}$ denotes the probability of visiting node $j$ in subgraph $g_i$ starting from the center node $v_c$ via a random walk. $\mathbf{z}_{g,i,r}^{s}$ is the structural representation of the $r$-th perturbed view of $g_i$, $\mathcal{V}_{t}(g_i)$ denotes nodes with observed textual attributes, and $\mathcal{V}(g_i)$ denotes all nodes in the sampled graph object.
The final sparse reliability score is
\begin{equation}
    \rho_i=\mathrm{clip}(\lambda_dd_i+\lambda_aa_i+\lambda_cc_i,0,1)
    \label{eq:sparse_reliability}
\end{equation}
where $\lambda_d$, $\lambda_a$, and $\lambda_c$ control the contribution of the three reliability factors.
A sample receives a higher reliability score when it exhibits more stable topology, more complete textual attributes, and stronger structure--semantic consistency.

\textbf{Adaptive sample weighting.}
We combine the density ratio and sparse reliability into a unified sample weight.
For each mini-batch domain subset $\widetilde{\mathcal{T}}_e=\widetilde{\mathcal{T}}\cap\mathcal{D}_e$, the normalized weight is
\begin{equation}
    \omega_i
    =
    \frac{
    |\widetilde{\mathcal{T}}_{e(i)}|
    r_i\rho_i
    }{
    \sum_{(g_j,t_j)\in\widetilde{\mathcal{T}}_{e(i)}} r_j\rho_j
    },
    \qquad
    (g_i,t_i)\in\widetilde{\mathcal{T}}_{e(i)}.
    \label{eq:scrb_weight}
\end{equation}
This intra-domain normalization keeps the average weight within each source domain close to one, preventing domains with larger density values or higher textual completeness from dominating the training objective.

Built upon the enhanced per-sample contrastive loss $\ell_i^{+}$ derived in Eq.~\eqref{eq:sample_align}, we further formulate the weighted empirical risk for each domain $\mathcal{D}_e$ as
\begin{equation}
    R_e
    =
    \frac{
    \sum_{(g_i,t_i)\in\widetilde{\mathcal{T}}_e}
    \omega_i \ell_i^{+}
    }{
    \sum_{(g_i,t_i)\in\widetilde{\mathcal{T}}_e}
    \omega_i
    },
    \qquad
    \bar{R}
    =
    \frac{1}{K}
    \sum_{e=1}^{K}
    R_e .
    \label{eq:domain_risk}
\end{equation}

\textbf{Cross-domain risk balancing.}
To encourage the model to learn domain-invariant graph--text alignment rather than sparse domain-specific shortcuts, we penalize the dispersion of weighted risks across source domains:
\begin{equation}
    \mathcal{L}_{\mathrm{risk}}
    =
    \frac{1}{K}
    \sum_{e=1}^{K}
    \left(
    R_e-\bar{R}
    \right)^2 .
    \label{eq:risk_balance}
\end{equation}
Unlike pairwise domain constraints, this variance-style regularizer provides a compact objective for aligning all source-domain risks.

Finally, the Sparse-aware Cross-domain Risk Balancing mechanism is plugged into Eq.~\eqref{eq:base_loss} by using the weights $\boldsymbol{\omega}=\{\omega_i\}_{(g_i,t_i)\in\widetilde{\mathcal{T}}}$. The overall objective is
\begin{equation}
    \mathcal{L}_{\mathrm{total}}
    =
    \mathcal{L}_{\mathrm{align}}^{+}
    \left(
    \widetilde{\mathcal{T}};
    \boldsymbol{\omega}
    \right)
    +
    \alpha
    \mathcal{L}_{\mathrm{rec}}
    +
    \mu
    \mathcal{L}_{\mathrm{risk}}
    +
    \nu
    \mathcal{L}_{\mathrm{den}},
    \label{eq:total_loss}
\end{equation}
where $\alpha$, $\mu$, and $\nu$ control structure-oriented reconstruction, cross-domain risk balancing, and density estimation, respectively.

\subsection{Theoretical Analysis}
\label{sec:theory}

We provide a theoretical justification for the risk balancing mechanism from the perspective of weighted risk equalization.
Invariant learning methods~\cite{arjovsky2019invariant,peters2016causal,VREX} reduce out-of-domain generalization error by encouraging different source domains to share similar risks.
However, when graph--text supervision is sparse, direct risk matching may still be biased by domain-specific structural distributions and unreliable text attributes.
SCRB addresses this issue by using the global-domain density ratio to align shared structural support and using sparse reliability to suppress unreliable samples.

\begin{theorem}[Sparse-aware weighted risk equalization]
\label{thm:scrb}
Let $\mathcal{D}_1,\ldots,\mathcal{D}_K$ be $K$ source domains.
For a graph--text pair $(g,t)\sim P_e$, let $\mathbf{z}=\phi(g)$ denote its structural key, where $\phi(\cdot)$ corresponds to the graph-side structural representation used in Eq.~\eqref{eq:density_ratio}.
Let $p_e(\mathbf{z})$ be the marginal density of $\mathbf{z}$ in domain $\mathcal{D}_e$, and let
\begin{equation}
    p_0(\mathbf{z})
    =
    \sum_{e=1}^{K}
    \pi_e p_e(\mathbf{z}),
    \qquad
    r_e(\mathbf{z})
    =
    \frac{
    p_0(\mathbf{z})
    }{
    p_e(\mathbf{z})
    },
    \label{eq:theory_ratio}
\end{equation}
where $\pi_e>0$ and $\sum_{e=1}^{K}\pi_e=1$.
Assume that all domains share the same support over $\mathbf{z}$, and that the conditional reliability-weighted loss is invariant across domains:
\begin{equation}
    \mathbb{E}_{(g,t)\sim P_e}
    \left[
    \rho(g,t)\ell(g,t)
    \mid
    \mathbf{z}
    \right]
    =
    m(\mathbf{z}),
    \qquad
    \forall e\in\{1,\ldots,K\}.
    \label{eq:conditional_invariance}
\end{equation}
Then the density-ratio weighted risks are equal across all source domains:
\begin{equation}
    R_1^{\star}
    =
    R_2^{\star}
    =
    \cdots
    =
    R_K^{\star},
    \qquad
    R_e^{\star}
    =
    \mathbb{E}_{(g,t)\sim P_e}
    \left[
    r_e(\mathbf{z})
    \rho(g,t)
    \ell(g,t)
    \right].
    \label{eq:weighted_equalization}
\end{equation}
\end{theorem}
\begin{proof}
    The detailed proof is provided in Appendix~\ref{app:proof-theorem-3-1}.
\end{proof}

\section{Experiments}
\vspace{-2mm}
\label{Experiments}

\subsection{Experimental Setup}
\vspace{-1mm}
\textbf{Datasets.}
We evaluate \ourmodel on text-attributed graph datasets from academic, e-commerce, social, and Wikipedia domains. Following the graph--text pre-training setting, ogbn-ArXiv~\cite{hu2020open}, ArXiv\_2023~\cite{he2023explanations}, PubMed~\cite{sen2008collective}, ogbn-Products~\cite{hu2020open}, and Reddit~\cite{huang2024can} are used as source datasets, while Cora~\cite{sen2008collective}, CiteSeer~\cite{sen2008collective}, Books-History~\cite{yan2023comprehensive}, Ele-Computers~\cite{yan2023comprehensive}, Ele-Photo~\cite{yan2023comprehensive}, WikiCS~\cite{mernyei2020wiki}, and Instagram~\cite{huang2024can} are used as unseen target datasets. To simulate sparse TAG scenarios, we retain only 1\%, 3\%, 5\%, or 10\% of node textual attributes during pre-training. Unless otherwise specified, experiments are conducted under the 10\% sparsity setting.

\textbf{Baselines.}
We compare \ourmodel with three groups of baselines: 
(1) LLM-only methods, including SBERT~\cite{reimers2019sentence} and Qwen3-0.6B~\cite{yang2025qwen3}, which use textual information without explicitly modeling graph structure; 
(2) representative TAG methods, including GraphGPT~\cite{tang2024graphgpt}, LLaGA~\cite{chen2024llaga}, OFA~\cite{liu2023one}, ZeroG~\cite{li2024zerog}, ADAligner~\cite{adaligner}, and GraphCLIP~\cite{zhu2025graphclip}; 
and (3) self-supervised graph methods applied to TAGs, including DGI~\cite{velickovic2018deep}, GRACE~\cite{zhu2020deep}, BGRL~\cite{thakoor2021bootstrapped}, GraphMAE~\cite{hou2022graphmae}, and G2P2~\cite{G2P2}.

\textbf{Evaluation and Implementation.}
We evaluate zero-shot node classification, link prediction, and cross-modal retrieval. We report accuracy for node classification, AUC for link prediction, and MRR/Recall@K for retrieval. All results are averaged over five random seeds and reported as mean $\pm$ standard deviation. For \ourmodel, we use a graph transformer-based encoder and a frozen pre-trained language encoder. The trainable modules are optimized with graph--text contrastive learning, structure-oriented reconstruction, and reliability-aware robust training.

\subsection{RQ1: Zero-Shot Inference on Target Data}
\vspace{-1mm}
\label{Q1}
We first evaluate whether \ourmodel can generalize to unseen target graphs without task-specific fine-tuning. After pre-training on sparse source TAGs, the model is directly applied to target datasets for zero-shot inference. We consider three downstream tasks—node classification, link prediction, and cross-modal retrieval—which respectively assess semantic label prediction, topology-aware transfer, and fine-grained graph--text alignment.

\begin{table*}[t]
\centering
\caption{Zero-shot node classification accuracy (\%) on unseen target datasets under the 10\% text sparsity setting. Results are averaged over five random seeds.}
\label{tab:node-classification}
\resizebox{\textwidth}{!}{
\begin{tabular}{lcccccccc}
\toprule
\multirow{2}{*}{Method} 
& \multicolumn{2}{c}{Academic} 
& \multicolumn{3}{c}{E-commerce} 
& \multicolumn{2}{c}{Web}
& \multirow{2}{*}{Avg.} \\
\cmidrule(lr){2-3} \cmidrule(lr){4-6} \cmidrule(lr){7-8}
& Cora & CiteSeer & Books-History & Ele-Computers & Ele-Photo & WikiCS & Instagram & \\
\midrule
SBERT & 59.67 $\pm$ 1.76 & \underline{66.93 $\pm$ 1.82} & 41.18 $\pm$ 0.56 & 41.01 $\pm$ 0.26 & 38.36 $\pm$ 0.44 & \underline{60.87 $\pm$ 0.01} & 59.91 $\pm$ 0.73 & 52.56 \\
Qwen3-0.6B & 41.48 $\pm$ 1.66 & 51.13 $\pm$ 0.95 & 39.31 $\pm$ 0.61 & 12.35 $\pm$ 0.09 & 32.87 $\pm$ 0.40 & 40.94 $\pm$ 0.01 & \textbf{64.47} $\pm$ 0.63 & 40.36 \\
\midrule
GraphGPT 
& 14.45$\pm$0.37 & 18.38$\pm$0.32 
& \underline{55.58$\pm$0.03} & 17.29$\pm$0.06 & 40.89$\pm$0.02 
& 4.42$\pm$0.05 & 45.94$\pm$0.28 & 28.14 \\
LLaGA 
& 30.15$\pm$1.88 & 35.72$\pm$1.01 
& 45.48$\pm$0.92 & 28.36$\pm$0.85 & 25.91$\pm$0.47 
& 30.27$\pm$0.56 & 48.63$\pm$1.14 & 34.93 \\
OFA 
& 8.09$\pm$0.01 & 20.53$\pm$0.01 
& 7.86$\pm$0.01 & 7.40$\pm$0.07 & 28.45$\pm$0.07 
& 12.74$\pm$0.07 & 41.64$\pm$0.14 & 18.10 \\
ZeroG & 53.47{$\pm$1.19} & 50.16{$\pm$0.01} & 35.69{$\pm$0.49} & 33.79{$\pm$0.18} & 36.23{$\pm$0.50} & 51.29{$\pm$0.01} & 52.56{$\pm$0.67} & 44.74 \\
\midrule
DGI 
& 27.36$\pm$1.08 & 33.27$\pm$0.15 
& 13.39$\pm$0.33 & 13.86$\pm$0.70 & 12.49$\pm$0.64 
& 46.85$\pm$3.42 & 58.63$\pm$0.16 & 29.41 \\
GRACE 
& 34.90$\pm$0.51 & 28.69$\pm$1.98 
& 24.42$\pm$1.86 & 20.07$\pm$1.27 & 21.56$\pm$0.99 
& 55.93$\pm$0.99 & 46.75$\pm$2.68 & 33.19 \\
BGRL 
& 13.52$\pm$1.74 & 17.41$\pm$2.90 
& 6.85$\pm$2.86 & 13.45$\pm$7.96 & 7.95$\pm$4.05 
& 9.42$\pm$6.41 & 46.13$\pm$8.49 & 16.39 \\
GraphMAE 
& 11.29$\pm$3.89 & 16.77$\pm$4.86 
& 11.72$\pm$22.27 & 6.64$\pm$5.68 & 11.05$\pm$15.36 
& 9.82$\pm$3.96 & \underline{63.71$\pm$0.01}& 18.71 \\
G2P2 
& 29.78$\pm$3.04 & 34.85$\pm$6.09 
& 17.28$\pm$3.68 & 20.25$\pm$4.39 & 20.31$\pm$5.86 
& 28.32$\pm$3.84 & 53.47$\pm$3.31 & 29.18 \\
\midrule
ADAligner 
& \underline{60.51$\pm$1.64} & 64.52$\pm$1.77 
& 51.94$\pm$1.70 & 50.45$\pm$1.24 & \underline{41.06$\pm$0.92} 
& 57.65$\pm$3.01 & 54.84$\pm$1.57 & 54.42 \\
GraphCLIP 
& 57.68$\pm$0.70 & 63.73$\pm$1.91 
& 49.79$\pm$3.02 & \underline{59.09$\pm$1.84} & 36.78$\pm$2.73 
& 58.58$\pm$3.78 & 60.48$\pm$2.03 & \underline{55.16} \\
\rowcolor{gray!15}
\ourmodel 
& \textbf{66.73$\pm$1.40} & \textbf{67.25$\pm$2.32	} 
& \textbf{56.00$\pm$1.01} & \textbf{60.18$\pm$2.89} & \textbf{44.01$\pm$1.09} 
& \textbf{62.29$\pm$3.61} & 62.71$\pm$0.21 & \textbf{59.88} \\
\bottomrule
\end{tabular}}
\end{table*}
	
\begin{table*}[!t]
\centering
\begin{minipage}[!t]{0.44\textwidth}
\centering
\caption{Zero-shot link prediction performance measured by AUC.}
\label{tab:link-prediction}
\scriptsize
\setlength{\tabcolsep}{3pt}
\resizebox{\linewidth}{!}{
\begin{tabular}{lccc}
\toprule
Method & Cora & WikiCS & History \\
\midrule
SBERT & 83.11 {\tiny $\pm$ 0.41} & 78.91 {\tiny $\pm$ 1.49} & 70.66 {\tiny $\pm$ 0.23} \\
OFA & 58.90 {\tiny $\pm$ 0.28} & 59.82 {\tiny $\pm$ 0.12} & 63.97 {\tiny $\pm$ 0.13} \\
ZeroG & 84.65 {\tiny $\pm$ 0.94} & 85.18 {\tiny $\pm$ 0.14} & 83.61 {\tiny $\pm$ 0.12} \\
GraphMAE & 55.72 {\tiny $\pm$ 3.86} & 69.02 {\tiny $\pm$ 0.47} & 76.54 {\tiny $\pm$ 0.43} \\
GraphCLIP & 68.91 {\tiny $\pm$ 2.85} & 66.02 {\tiny $\pm$ 6.46} & 89.00 {\tiny $\pm$ 2.98} \\
ADAligner & 77.44 {\tiny $\pm$ 2.32} & 73.33 {\tiny $\pm$ 1.85} & 70.45 {\tiny $\pm$ 2.52} \\
\rowcolor{gray!15}
\ourmodel & \textbf{91.06}{\tiny $\pm$0.09} & \textbf{87.15}{\tiny $\pm$0.11} & \textbf{91.76}{\tiny $\pm$0.12} \\
\bottomrule
\end{tabular}}
\end{minipage}
\hfill
\begin{minipage}[!t]{0.54\textwidth}
\centering
\caption{Cross-modal retrieval results under global-level and category-level settings.}
\label{tab:retrieval-global}
\scriptsize
\setlength{\tabcolsep}{2.5pt}
\resizebox{\linewidth}{!}{
\begin{tabular}{lllcccc}
\toprule
Setting & Method & Task & MRR & R@1 & R@5 & R@10 \\
\midrule
Global & GraphCLIP & N2T
& 7.56{\tiny $\pm$1.92}
& 4.62{\tiny $\pm$1.31}
& 10.07{\tiny $\pm$2.71}
& 13.70{\tiny $\pm$3.36} \\
& & T2N
& 38.85{\tiny $\pm$1.78}
& 33.63{\tiny $\pm$1.77}
& 44.50{\tiny $\pm$1.81}
& 48.88{\tiny $\pm$2.14} \\
\rowcolor{gray!15}
& \ourmodel & N2T
& \textbf{9.84}{\tiny $\pm$3.64}
& \textbf{5.74}{\tiny $\pm$2.69}
& \textbf{13.47}{\tiny $\pm$4.58}
& \textbf{18.23}{\tiny $\pm$5.64} \\
\rowcolor{gray!15}
& & T2N
& \textbf{95.62}{\tiny $\pm$0.27}
& \textbf{93.13}{\tiny $\pm$0.45}
& \textbf{98.64}{\tiny $\pm$0.23}
& \textbf{99.24}{\tiny $\pm$0.14} \\
\midrule
Category & GraphCLIP & N2T
& 19.75{\tiny $\pm$2.45}
& 12.09{\tiny $\pm$2.16}
& 26.79{\tiny $\pm$2.96}
& 35.59{\tiny $\pm$3.41} \\
& & T2N
& 44.33{\tiny $\pm$1.64}
& 38.21{\tiny $\pm$1.64}
& 50.78{\tiny $\pm$1.97}
& 48.88{\tiny $\pm$2.14} \\
\rowcolor{gray!15}
& \ourmodel & N2T
& \textbf{20.99}{\tiny $\pm$3.94}
& \textbf{12.60}{\tiny $\pm$3.35}
& \textbf{28.92}{\tiny $\pm$4.91}
& \textbf{38.17}{\tiny $\pm$5.16} \\
\rowcolor{gray!15}
& & T2N
& \textbf{96.79}{\tiny $\pm$0.23}
& \textbf{94.89}{\tiny $\pm$0.26}
& \textbf{99.05}{\tiny $\pm$0.34}
& \textbf{99.49}{\tiny $\pm$0.16} \\
\bottomrule
\end{tabular}}
\end{minipage}
\end{table*}
\textbf{Node Classification.}
Table~\ref{tab:node-classification} reports the zero-shot node classification accuracy on seven target datasets under the \(10\%\) text sparsity pre-training setting. \ourmodel achieves the best overall performance among all compared methods and shows consistent improvements over strong graph--text alignment baselines. Compared with GraphCLIP, \ourmodel improves the average accuracy from \(55.16\%\) to \(59.88\%\), demonstrating stronger transferability under sparse textual supervision. The gains are especially evident on Cora, CiteSeer, Books-History, and Ele-Photo. We further observe that our learned node representations better capture the semantic distinctions among nodes compared with GraphCLIP, for a visual comparison, see Appendix~\ref{Visualization}.

\textbf{Link Prediction.} We further evaluate zero-shot link prediction for transferable structural information. The pre-trained model is applied to target datasets without task-specific fine-tuning. For link prediction, we report average AUC scores and standard deviations over five runs with different random seeds. As shown in Table~\ref{tab:link-prediction}, \ourmodel achieves strong zero-shot performance. On Cora, WikiCS, and History, it obtains the best AUC scores, outperforming baselines.

\textbf{Cross-modal Retrieval.}We further evaluate cross-modal retrieval under both global-level and category-level settings, including node-to-text (N2T) and text-to-node (T2N) retrieval on the WikiCS dataset. As shown in Table~\ref{tab:retrieval-global}, \ourmodel consistently outperforms GraphCLIP across almost all retrieval metrics in both settings. The improvement is particularly significant for T2N retrieval, where \ourmodel achieves much higher MRR and Recall scores, indicating that the learned text representations can more accurately retrieve the corresponding graph instances. Under the category-level setting, \ourmodel also maintains consistent gains, suggesting that the proposed sparse structural alignment improves not only global graph--text matching but also fine-grained category-aware retrieval. 

\begin{figure*}[t]
\centering
\begin{minipage}[t]{0.32\textwidth}
\centering
\includegraphics[width=\linewidth]{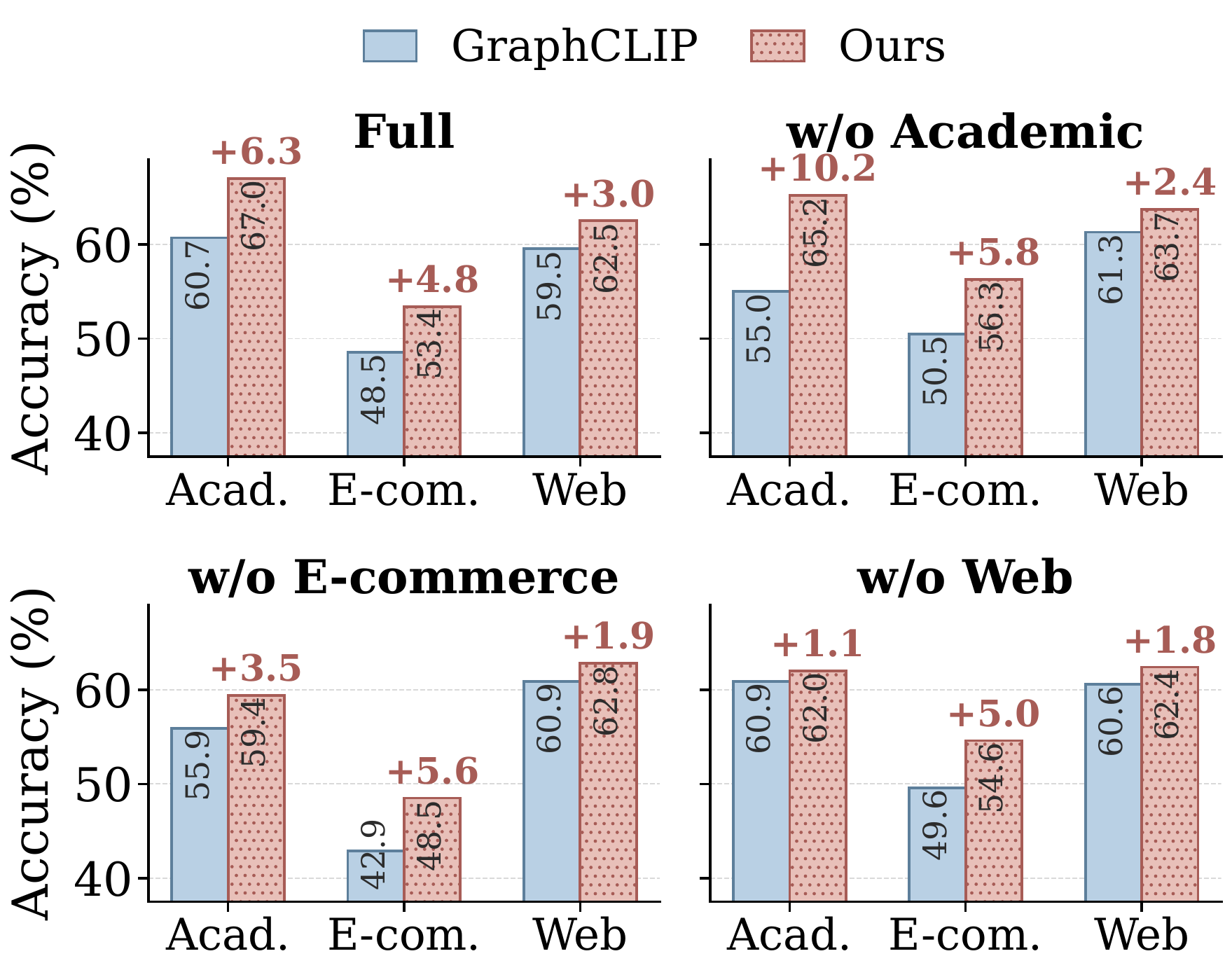}
\caption{Effect of source-domain composition on zero-shot transfer.}
\label{fig:source-domain}
\end{minipage}
\hfill
\begin{minipage}[t]{0.33\textwidth}
\centering
\includegraphics[width=\linewidth]{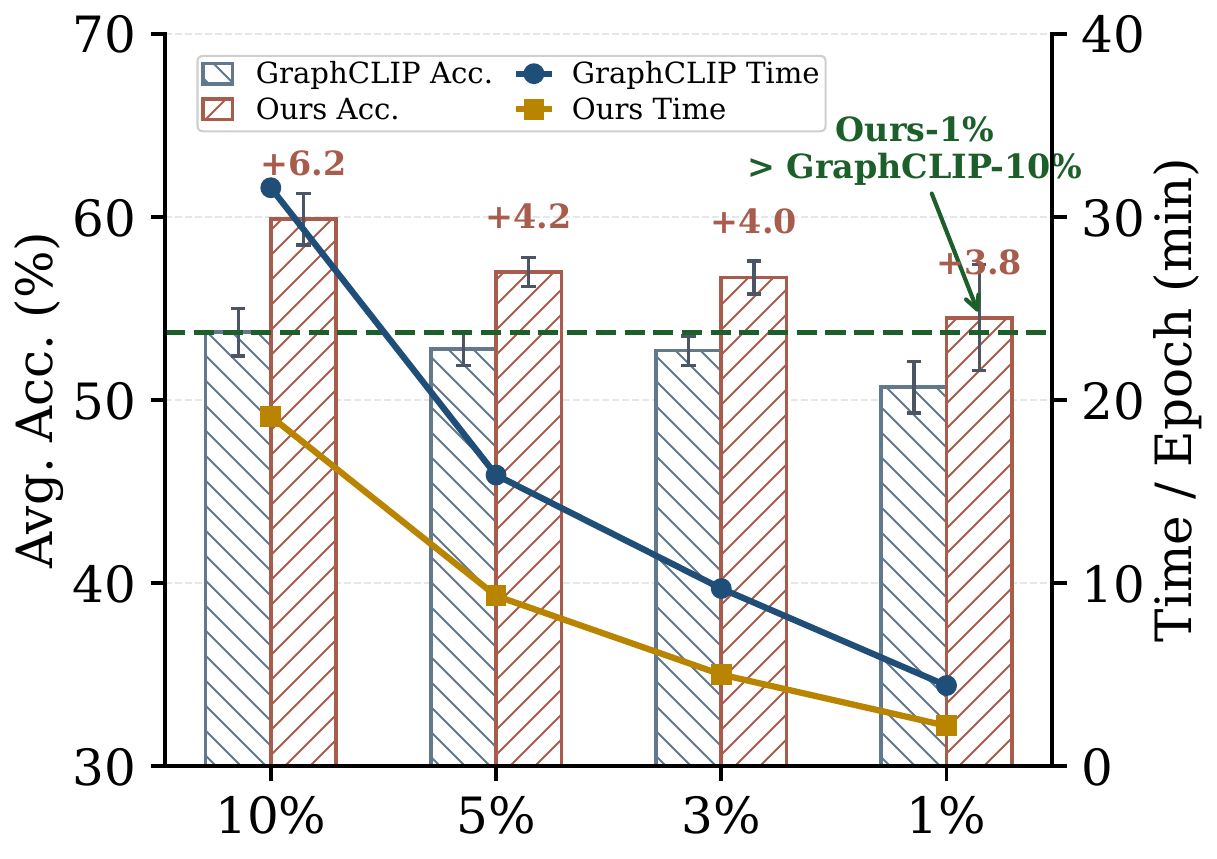}
\caption{Performance-efficiency trade-off under varying text sparsity levels.}
\label{fig:sparsity-efficiency}
\end{minipage}
\hfill
\begin{minipage}[t]{0.32\textwidth}
\centering
\includegraphics[width=\linewidth]{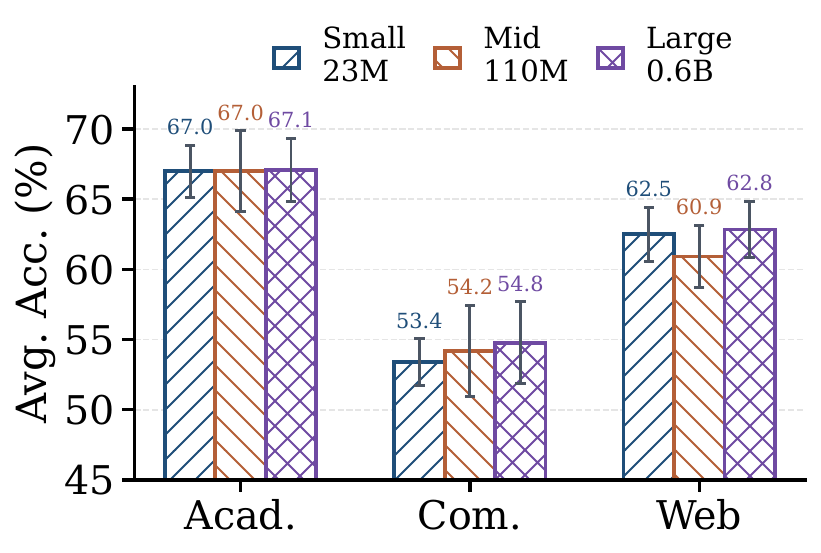}
\caption{Effect of language encoder scale on zero-shot transfer.}
\label{fig:model-scale}
\end{minipage}
\end{figure*}

\subsection{RQ2: Explorations on Source Datasets}
\label{Q2}
To investigate how source-domain composition affects transferability, we compare \ourmodel with GraphCLIP under different source dataset configurations. Following GraphCLIP, we consider four settings: using all source datasets, removing academic datasets, removing e-commerce datasets, and removing web datasets. We report the average target accuracy grouped by target domain, including academic, e-commerce, and web domains. As shown in Figure~\ref{fig:source-domain}, \ourmodel consistently outperforms GraphCLIP across all source-domain configurations and target-domain groups. The gains remain clear even when one source domain is removed, indicating that \ourmodel does not simply rely on domain-specific shortcuts from the source data. Instead, the proposed sparsity-aware structural alignment encourages more transferable graph--text representations that remain effective under reduced source-domain diversity.

\subsection{RQ3: Zero-Shot Robustness \& Efficiency Experiment }
\label{Q3}
\vspace{-1mm}
\textbf{Varying Text Sparsity: Performance and Efficiency.} 
We further evaluate \ourmodel under text sparsity levels of 1\%, 3\%, 5\% and 10\%, aiming to explore whether promising graph-text alignment can be preserved when only limited node textual attributes are accessible during pre-training. As illustrated in Figure~\ref{fig:sparsity-efficiency}, \ourmodel steadily outperforms GraphCLIP under all sparsity configurations; notably, the model trained merely with 1\% textual attributes achieves even better performance than GraphCLIP trained with 10\% textual attributes, demonstrating superior sparsity generalization and higher data efficiency. Apart from consistent performance gains, our method also achieves substantial computational acceleration.The proposed approach considerably reduces per-epoch training overhead, yielding markedly lower runtime cost compared with GraphCLIP. These observations confirm that our framework delivers competitive zero-shot transfer performance with limited textual attributes and lower computational overhead, making it well-suited for large-scale sparse text-attributed graph scenarios. Detailed time and space complexity analysis of \ourmodel is provided in the Appendix~\ref{Complexity}.

\textbf{Varying Language Model Scale.} 
We evaluate the effect of language encoder scale by replacing the text encoder with MiniLM-L6~\cite{wang2020minilm}, E5~\cite{wang2022text}, and Qwen3-0.6B~\cite{yang2025qwen3} (small, medium, and large). As shown in Figure~\ref{fig:model-scale}, performance differences are minor, with the small encoder already competitive. While the large encoder slightly improves some domains, overall gains are limited, indicating that \ourmodel's transferability is primarily driven by its sparsity-aware, structure-enhanced alignment rather than encoder scale.

\textbf{Hyperparameter Sensitivity.} 
\ourmodel shows stable performance across a wide range of hyperparameter settings, including $\alpha$ (structural reconstruction), $\mu$ (cross-domain risk balance), and $\nu$ (density estimation), consistently outperforming GraphCLIP (detailed results in Appendix~\ref{app:HS}). We also report the fixed choice of the three reliability factors, $\lambda_d$, $\lambda_a$, and $\lambda_c$.

\subsection{RQ4: Ablation Study}
\vspace{-2mm}
\label{Q5}
\begin{figure}[!h]
    \centering
    \begin{minipage}{0.6\linewidth}
        \vspace{0pt}
        As shown in Figure~\ref{fig:ablation}, removing any key component leads to performance degradation, confirming that all modules contribute to the final results. Specifically, \textbf{w/o Structural Reconstruction} drops notably on the e-commerce and Web domains, highlighting the importance of topology-aware auxiliary supervision under sparse textual conditions; \textbf{w/o Risk Balancing} also reduces performance, indicating that sparse-aware cross-domain risk balancing helps mitigate domain-biased alignment and improve transferability; \textbf{w/o Summary} consistently degrades performance, particularly on the academic domain, demonstrating that semantic summaries are essential for stable graph--text alignment.

    \end{minipage}
    \hfill
    \begin{minipage}{0.38\linewidth}
        \vspace{0pt}
        \centering
        \includegraphics[width=\linewidth]{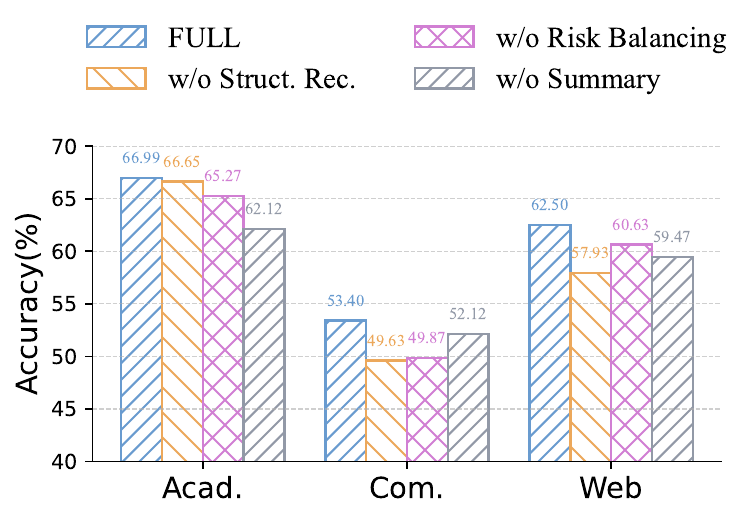}
        \caption{The study results.}
        \label{fig:ablation}
    \end{minipage}
\end{figure}

\vspace{-1mm}
\section{Related Work}
\vspace{-2mm}
\label{RelatedWork}

\noindent \textbf{Graph–Text Alignment}. 
Inspired by CLIP~\cite{clip}, contrastive dual-encoder frameworks have become the dominant paradigm for graph--text alignment on text-attributed graphs. Methods such as GraphCLIP~\cite{zhu2025graphclip}, G2P2~\cite{G2P2}, and GRENADE~\cite{li2023grenade} construct graph-text positive pairs and map them into a shared space. However, they rely on fixed one-to-one alignment, limiting their ability to capture many-to-many semantic relations. ConGraT~\cite{brannon2024congrat} introduces graph similarity to generate soft alignment targets, and ADAligner~\cite{adaligner} dynamically balances one-to-one and many-to-many objectives based on batch-level data quality. These approaches, however, assume reliable node texts and summaries; in sparse scenarios, missing textual attributes weaken semantic evidence and may mislead supervision.

\noindent \textbf{Cross-domain Generalization}.
Cross-domain generalization aims to learn models from multiple source domains that can transfer to unseen target domains. Its core objective is to reduce the model's reliance on domain-specific correlations and capture stable predictive mechanisms across domains. Existing studies mitigate distribution shifts from different perspectives, including domain-adversarial learning \cite{DANN}, distribution alignment \cite{DAN,CORAL}, invariant risk minimization \cite{arjovsky2019invariant}, risk variance regularization \cite{VREX}, and distributionally robust optimization \cite{GroupDRO}. In graph learning, recent works have further investigated OOD generalization on graphs. For example, GIL~\cite{GIL} captures stable structural information through invariant subgraph learning, while EERM~\cite{EERM} improves graph OOD generalization through environment exploration. However, sparse graph--text pre-training faces a more complex scenario, where textual attribute missingness, structural variation, and source-domain distribution shifts jointly affect graph--text alignment.

\section{Conclusion}
\vspace{-2mm}
\label{conclusion}
This paper studied graph--text pre-training on sparse text-attributed graphs and presented \ourmodel, a sparsity-aware and structure-enhanced LLM-as-Aligner framework. \ourmodel disentangled content and structural representations, introduced structure-oriented reconstruction, and applied sparse-aware cross-domain risk balancing to improve robust graph--text alignment. Experiments on zero-shot node classification, link prediction, and cross-modal retrieval demonstrated that \ourmodel achieves strong transferability across sparse settings and unseen target graphs.

\bibliographystyle{plain}
\bibliography{ref.bib}
\newpage
\appendix
\appendix

\section{Appendix Overview}
The appendix is structured as follows:
\begin{itemize}[leftmargin=*]
    \item Section~\ref{limits} discusses limitations and future work.
    \item Section~\ref{impact} provides the broader impact statement.
    \item Section~\ref{app:theoretical} presents the theoretical analysis of the proposed global-domain density ratio weighting.
    \item Section~\ref{Datasets} details the experimental setup, including datasets and baselines.
    \item Section~\ref{sec:implementation_details} provides implementation details of \ourmodel, including pre-training settings and zero-shot prompts.
    \item Section~\ref{app:exp} reports additional experimental results, including hyperparameter sensitivity and visualization analysis.
    \item Section~\ref{Complexity} analyzes the time and space complexity of S$^2$Aligner.
\end{itemize}

\section{Limitations and Future Work}
\label{limits}
Despite the effectiveness of our proposed \ourmodel in capturing topological cues and facilitating text-graph alignment via the structure reconstruction consistency constraint, the current structural enhancement paradigm still performs optimally on sparsely text-attributed graphs with regular and clear topological connections. When facing highly fragmented graph structures with numerous isolated nodes and ambiguous topological relationships, the useful structural cues can be easily contaminated by redundant noises, leading to the dilution of structural enhancement signals and thus degenerating the robustness of cross-modal alignment. To mitigate this limitation, future work will explore adaptive topology-aware modeling and structure denoising strategies to capture latent connections in fragmented graph structures, and dynamically adjust the strength of structural constraints according to the degree of fragmentation, thereby improving the alignment performance on disordered and complex topologies.

\section{Impact Statements}
\label{impact}
This work contributes to the development of graph--text pre-training by improving robust graph--text alignment on sparse text-attributed graphs. The proposed framework can be applied to low-resource graph--text scenarios, such as academic network analysis, product graph understanding, recommendation systems, and social information analysis. It may help reduce the reliance on large-scale manually annotated data and improve the usability and generalization ability of graph--text models in practical applications. Since this work mainly focuses on methodological and technical improvements, and the datasets used in our experiments do not involve sensitive personal information or privacy risks, we do not expect the proposed method to have obvious negative societal impacts.

\section{Theoretical Analysis}
\label{app:theoretical}

\subsection{Proof of Theorem 3.1}
\label{app:proof-theorem-3-1}

This proof follows the standard paradigm of weighted risk invariance and general position in invariant learning~\cite{WRI}.
In contrast to the pairwise weighting strategy that uses a single reference domain, we propose a \emph{global-domain density ratio weighting} derived from a unified global mixture distribution, which serves as an unbiased and stable anchor for multi-domain generalization.

\paragraph{Notation}
\begin{itemize}[leftmargin = 16pt]
    \item $Z$: Shared invariant structural component across all domains.
    \item $V_e$: Domain-private spurious component of the $e$-th domain.
    \item $p_e(Z)$: Marginal density of invariant feature $Z$ in domain $e$.
    \item $p_0(Z) = \sum_e \pi_e p_e(Z)$: Global mixture distribution.
    \item Our global weighting function:
    \begin{equation*}
        r_e(Z) \triangleq \frac{p_0(Z)}{p_e(Z)}.
    \end{equation*}
    \item Linear predictor:
    \begin{equation*}
        f(Z, V_e) = \mathbf{w}^\top \begin{bmatrix} Z^\top & V_e^\top \end{bmatrix}^\top.
    \end{equation*}
    \item Parameter decomposition:
    \begin{equation*}
        \mathbf{w} = \begin{bmatrix} \mathbf{w}_{\text{inv}}^\top & \mathbf{w}_{\text{spu}}^\top \end{bmatrix}^\top,
    \end{equation*}
    where $\mathbf{w}_{\text{inv}}$ corresponds to invariant features $Z$, and $\mathbf{w}_{\text{spu}}$ corresponds to domain-private spurious features $V_e$.
\end{itemize}

\paragraph{Assumptions}
Let $\mathcal{E}$ denote the set of training domains. We assume:
\begin{enumerate}[leftmargin = 16pt]
    \item Feature decomposition: inputs consist of invariant $Z$ and domain-specific $V_e$.
    \item Conditional independence: $Y \perp\!\!\!\perp V_e \mid Z$.
    \item Invariant label condition: $p_{e_i}(Y \mid Z) = p_{e_j}(Y \mid Z)$ for all $e_i,e_j \in \mathcal{E}$.
    \item Non-degenerate density: $p_e(Z) > 0$ for all $e \in \mathcal{E}$.
\end{enumerate}

\paragraph{Proposition 1 }
If the assumptions hold and the predictor $f$ depends only on $Z$, then the weighted risks are equal across all domains.

\begin{proof}
The weighted risk for domain $e$ is:
\begin{equation*}
    R_e = \iiint_{\mathcal{Y}\mathcal{Z}\mathcal{V}_e} p_e(Z,V_e,Y) \cdot \ell(f(Z),Y) \cdot r_e(Z) \,\mathrm{d}V_e \mathrm{d}Z \mathrm{d}Y.
\end{equation*}
Marginalizing out $V_e$:
\begin{equation*}
    R_e = \iint_{\mathcal{Y}\mathcal{Z}} \ell(f(Z),Y) \cdot r_e(Z) \cdot p_e(Z,Y) \,\mathrm{d}Z \mathrm{d}Y.
\end{equation*}
Substituting $p_e(Z,Y) = p_e(Y \mid Z)p_e(Z)$ and $r_e(Z) = p_0(Z)/p_e(Z)$:
\begin{equation*}
    R_e = \iint_{\mathcal{Y}\mathcal{Z}} \ell(f(Z),Y) \cdot p_0(Z) \cdot p_e(Y \mid Z) \,\mathrm{d}Z \mathrm{d}Y.
\end{equation*}
Since $p_0(Z)$ and $p_e(Y \mid Z)$ are domain-agnostic, $R_e$ is identical for all $e \in \mathcal{E}$.
\end{proof}

\paragraph{Definition 1 (General Position)}
Let $k$ be the number of domains such that $\binom{k}{2} > 2d_{\text{spu}}$. Define:
\begin{equation*}
\begin{aligned}
    \Sigma_{i,j} &= \mathbb{E}_{P_i}\!\left[ [Z,V_e][Z,V_e]^\top \cdot r_i(Z) \right], \\
    \mu_{i,j} &= 2 \cdot \mathbb{E}_{P_i}\!\left[ [Z,V_e] Y \cdot r_i(Z) \right].
\end{aligned}
\end{equation*}
Domains are in \emph{general position} if for any scalar $\gamma \in \mathbb{R}$ and nonzero $\mathbf{x} \in \mathbb{R}^{d_{\text{spu}}}$:
\begin{equation*}
    \dim\left(\mathrm{span}\!\left\{ (\Sigma_{i,j} - \Sigma_{j,i})\mathbf{x} + \gamma(\mu_{i,j} - \mu_{j,i})\mathbf{x} \right\}_{i,j \in [k]}\right) = d_{\text{spu}}.
\end{equation*}
This definition follows the standard formulation in invariant learning~\cite{arjovsky2019invariant,WRI}.

\paragraph{Theorem 1 (Elimination of Spurious Features)}
Under linear regression setting, if $|\mathcal{E}| > d_{\text{spu}}$ and domains are in general position, then any linear predictor satisfying global weighted risk invariance must have:
\begin{equation*}
    \mathbf{w}_{\text{spu}} = \mathbf{0}.
\end{equation*}
Our global-domain density ratio weighting satisfies general position with probability 1.

\begin{proof}
For any two domains, risk invariance gives:
\begin{equation*}
\begin{aligned}
    &\iiint \left( \mathbf{w}^\top [Z,V_e] - Y \right)^2 p_1 r_1 \,\mathrm{d}V_e \mathrm{d}Z \mathrm{d}Y \\
    =\ &\iiint \left( \mathbf{w}^\top [Z,V_e] - Y \right)^2 p_2 r_2 \,\mathrm{d}V_e \mathrm{d}Z \mathrm{d}Y.
\end{aligned}
\end{equation*}
Expanding leads to:
\begin{equation*}
    \mathbf{w}^\top \Sigma_{12} \mathbf{w} - \mathbf{w}^\top \mu_{12} + c_{12}
    = \mathbf{w}^\top \Sigma_{21} \mathbf{w} - \mathbf{w}^\top \mu_{21} + c_{21}.
\end{equation*}
By $r_e(Z)p_e(Z) = p_0(Z)$, invariant terms cancel, leaving constraints only on $\mathbf{w}_{\text{spu}}$.
Under general position, the system has full rank, so $\mathbf{w}_{\text{spu}} = \mathbf{0}$.

Following the measure-theoretic property for density ratio weights~\cite{WRI}, our global weighting satisfies general position with probability 1.
\end{proof}

\section{Experimental Setup Details}
\label{Datasets}
\subsection{Datasets}
To comprehensively assess the transfer capabilities and sparsity tolerance of our \ourmodel, we conduct cross-domain evaluations following the zero-shot protocol established by \cite{zhu2025graphclip}. The empirical study incorporates 12 distinct Text-Attributed Graphs (TAGs). These networks are partitioned into a set of 5 source domains for the pre-training phase, and 7 target domains for downstream verification.

\textbf{Source Datasets (Pre-training):} 
The pre-training corpus encompasses major online scenarios including academic literature, e-commerce retail, and social interactions. We adopt the complete topology for both academic and social networks, alongside a sampled subset of the retail graph to maintain domain balance.
\begin{itemize}[leftmargin = 8pt]
    \item \textbf{ogbn-ArXiv}~\cite{hu2020open}: Constructed from the Microsoft Academic Graph, this dataset forms a directed network mapping citations between computer science manuscripts. We utilize the available title and abstract information as raw features for each document. The downstream objective entails assigning these documents into 40 distinct topical categories.
    \item \textbf{ArXiv-2023}~\cite{he2023explanations}: Acting as a temporally updated counterpart to the aforementioned dataset, this collection captures recent publication citations from the year 2023 onwards. Models are similarly tasked with predicting the correct subject class out of 40 possible options for every scholarly article.
    \item \textbf{PubMed}~\cite{sen2008collective}: This graph organizes medical literature concerning diabetes research. Node entities must be classified into three specific clinical domains, separating experimental methodologies from Type 1 and Type 2 diabetes studies.
    \item \textbf{ogbn-Products}~\cite{hu2020open}: Sourced from Amazon's product listings, this network links items frequently bought together by consumers. Algorithms must accurately predict which of the 47 high-level retail categories a specific merchandise belongs to.
    \item \textbf{Reddit}~\cite{huang2024can}: This platform-derived graph captures online interactions, connecting individuals who have engaged in mutual conversation threads. Textual attributes originate from past user posts, and the core task involves predicting the popularity metric of the accounts.
\end{itemize}

\textit{Sparsity Setting:} To properly mimic the prevalent issue of missing information in real-world environments, we intentionally drop node attributes during the training stage. Under the default 10\% sparsity configuration, algorithms are only granted access to the raw texts of a minor fraction (10\%) of the total entities. The remaining vertices are strictly initialized with zero-vectors and ''unknown'' text flags, forcing the framework to heavily rely on topological propagation to infer absent semantics.

\textbf{Target Datasets (Downstream Evaluation):}
We verify the robust transferability of our model using 7 varied networks. During this phase, zero-shot inference is strictly conducted without applying any gradient updates on the target structures. For clarity, we categorize these datasets into three groups: academic citation networks (Cora, CiteSeer), product networks (Amazon-Computers, Amazon-Photo, Amazon-History), and web networks (WikiCS, Instagram). During this phase, zero-shot inference is strictly conducted without applying any gradient updates on the target structures.

\begin{itemize}[leftmargin = 8pt]
    \item \textbf{Cora}~\cite{sen2008collective}: A widely adopted academic benchmark containing 2,708 machine learning papers. References connect the manuscripts, and the model must identify one of seven corresponding research subfields.
    \item \textbf{CiteSeer}~\cite{sen2008collective}: Containing 3,186 documents spanning six computer science domains, this collection challenges the model to infer the correct research category by analyzing abstracts alongside paper titles.
    \item \textbf{WikiCS}~\cite{mernyei2020wiki}: Hyperlinks connect various computer science articles from Wikipedia to construct this structure. Algorithms leverage the main article texts to assign nodes into ten respective academic branches.
    \item \textbf{Amazon-Computers} and \textbf{Amazon-Photo}~\cite{yan2023comprehensive}: Both datasets emerge from Amazon's electronics inventory, where topological edges capture co-viewing or co-purchasing behaviors. Textual features incorporate highly-rated buyer reviews, demanding models to sort items into 10 and 12 distinct classes, respectively.
    \item \textbf{Amazon-History}~\cite{yan2023comprehensive}: Extracted specifically from the books category, relationships here also reflect consumer co-purchasing trends. We rely on book summaries and titles to predict a 12-way thematic classification.
    \item \textbf{Instagram}~\cite{huang2024can}: Nodes designate accounts on this social media platform, with following behaviors acting as structural edges. The core objective demands differentiating standard personal profiles from commercial entities.
\end{itemize}

\subsection{Baselines}
To validate the effectiveness of \ourmodel under extreme attribute scarcity, we conduct extensive comparisons against a robust suite of competitive algorithms. These existing solutions fall into four primary paradigms:

\textbf{1. Text-only Language Models:} 
These approaches independently process the raw sentences of each entity while completely ignoring topological connections.
\begin{itemize}[leftmargin = 8pt]
    \item \textbf{SBERT}~\cite{reimers2019sentence}: We incorporate standard sentence embedding architectures to produce dense semantic representations, specifically evaluating the \textit{all-MiniLM-L6-v2} and \textit{multi-qa-distilbert-cos-v1} variants.
    \item \textbf{Qwen3 (0.6B)}~\cite{yang2025qwen3}: A highly capable large language model acting as a robust generative baseline, tested on its ability to infer node labels using solely isolated contextual texts.
\end{itemize}

\textbf{2. Text-Attributed Graph (TAG) \& LLM-based Methods:}
These recent strategies strive to fuse structural patterns with the semantic comprehension capabilities of language models.
\begin{itemize}[leftmargin = 8pt]
    \item \textbf{GraphGPT}~\cite{tang2024graphgpt}: This architecture maps topological properties into discrete tokens and employs a dual-stage instruction fine-tuning process to synchronize GNN outputs with an LLM's semantic space.
    \item \textbf{LLaGA}~\cite{chen2024llaga}: By restructuring graph components into sequential prompts, this framework projects network topologies directly into the embedding layers of language models for generalized processing.
    \item \textbf{OFA}~\cite{liu2023one}: This prompting strategy textualizes diverse topological properties—including edges and neighbors—allowing a single LLM to handle varied graph-based tasks via natural language comprehension.
    \item \textbf{ZeroG}~\cite{li2024zerog}: To enable cross-domain generalization, this method converts structural tasks into text matching problems by integrating lightweight adapters alongside prompt-driven subgraph sampling.
\end{itemize}

\textbf{3. Graph Self-Supervised Learning (SSL) Models:}
Focusing primarily on topology and dense features, these methodologies leverage traditional graph neural networks.
\begin{itemize}[leftmargin = 8pt]
    \item \textbf{DGI}~\cite{velickovic2018deep}: A foundational self-supervised strategy that maximizes mutual information by distinguishing authentic node-graph representations from artificially corrupted counterparts.
    \item \textbf{GRACE}~\cite{zhu2020deep}: This contrastive learning approach generates augmented views of the same network, pulling identical node embeddings closer while pushing disparate nodes apart.
    \item \textbf{BGRL}~\cite{thakoor2021bootstrapped}: Drawing inspiration from BYOL architectures, this framework trains an online encoder against a momentum-updated target network, eliminating the necessity for negative sample pairs.
    \item \textbf{GraphMAE}~\cite{hou2022graphmae}: Functioning as a graph autoencoder, it deliberately conceals certain node attributes during the encoding phase and trains the decoder to accurately reconstruct those hidden features.
    \item \textbf{G2P2}~\cite{G2P2}: To improve performance in data-scarce scenarios, this technique blends structural pre-training signals with prompt-based text classification mechanisms.
\end{itemize}

\textbf{4. State-of-the-Art Graph-Text Aligners:}
Serving as our primary zero-shot competitors, these methods focus explicitly on synchronizing semantic and structural spaces.
\begin{itemize}[leftmargin = 8pt]
    \item \textbf{GraphCLIP}~\cite{zhu2025graphclip}: This framework relies heavily on contrastive alignment objectives. It synthesizes subgraph summaries to align embedding spaces, providing strong zero-shot graph--text alignment performance.
    
    \item \textbf{ADAligner}~\cite{adaligner}: ADAligner also employs contrastive alignment but integrates dynamic quality assessments to filter noisy signals during the alignment procedure, enhancing robustness against unreliable subgraph-text pairs.
\end{itemize}

\section{Implementation Details of \ourmodel}
\label{sec:implementation_details}

In this section, we provide the implementation details of \ourmodel. We first describe the experimental setup during the pre-training phase, and then present the prompts used for zero-shot learning.

\subsection{Pre-training Phase}

For \ourmodel, only a small number of hyperparameters need to be tuned. In our main experiments, we adopt the AdamW~\cite{adam} optimizer, with both the learning rate and weight decay set to \(1\times10^{-5}\). The graph encoder is implemented with GraphGPS~\cite{rampavsek2022recipe}, which consists of 12 layers with a hidden dimension of 1024. For the text encoder, we use a fine-tuned MiniLM~\cite{wang2020minilm} model with 6 layers and a hidden dimension of 384.

To align graph and text representations into a unified semantic space, we employ a projector to transform the 1024-dimensional graph representations into the same 384-dimensional space as the text representations. During pre-training, we optimize only the parameters of the graph encoder and the projector, while keeping the text encoder frozen to reduce training costs and mitigate catastrophic forgetting. The pre-training process is conducted for 10 epochs on an H100-80G GPU.

\subsection{Zero-shot Learning}

\textbf{Zero-shot Node Classification.} For zero-shot learning, we follow the same prompt construction strategy as the baseline GraphCLIP. Specifically, label information is incorporated into label-specific sentences to maintain consistency with the graph--text alignment format used during pre-training. Table~\ref{tab:zeroshot_prompts} presents the prompts designed for different datasets, where \(\{\text{class}\}\) denotes the label text of the target node, and \(\{\text{class\_desc}\}\) denotes the descriptive sentence generated by a large language model to further elaborate the semantic meaning of the label.

\begin{table}[t]
\centering
\caption{Prompts used for zero-shot learning of \ourmodel.}
\label{tab:zeroshot_prompts}
\begin{tabular}{ll}
\toprule
Dataset & Prompt \\
\midrule
Cora & \texttt{this paper has a topic on \{class\} \{class\_desc\}} \\
CiteSeer & \texttt{good paper of \{class\} \{class\_desc\}} \\
WikiCS & \texttt{it belongs to \{class\} research area \{class\_desc\}} \\
Instagram & \texttt{\{class\} \{class\_desc\}} \\
Ele-Photo & \texttt{this product belongs to \{class\} \{class\_desc\}} \\
Computers & \texttt{is \{class\} category \{class\_desc\}} \\
History & \texttt{this book belongs to \{class\} \{class\_desc\}} \\
\bottomrule
\end{tabular}
\end{table}

\textbf{Link Prediction.} 
For link prediction, the model computes the similarity between the source node embedding and the target node embedding. An edge is predicted to exist if the similarity score exceeds a threshold of 0.5. We evaluate the predictions using standard metrics such as AUC, F1, Precision, and Recall to quantify how well the predicted edges match the ground-truth connections in the graph.

\textbf{Node-Text Retrieval.} 
For node-text retrieval, we perform both \textit{Node-to-Text} and \textit{Text-to-Node} tasks. 
- In \textit{Node-to-Text} retrieval, each node embedding is used to rank all candidate text embeddings according to cosine similarity, and the top-k results are returned as predictions.  
- In \textit{Text-to-Node} retrieval, each text embedding is used to rank all candidate node embeddings similarly.  
Metrics such as Mean Reciprocal Rank (MRR) and Recall@k are reported to evaluate retrieval quality and the model's ability to semantically align nodes with text.

\section{Additional Results for Experiments}
\label{app:exp}
\subsection{Hyperparameter Sensitivity Analysis}
\label{app:HS}
We evaluate the sensitivity of our key hyperparameters \(\alpha\), \(\mu\) and \(\nu\), which control the weights of structural reconstruction loss, cross-domain risk balance loss and density estimation loss, respectively. We vary each hyperparameter in the range \(\{0.1,0.5,1.0,2.0,3.0\}\) and report the results with standard deviation.

\begin{figure}[H]
\centering
\subfigure[$\alpha$ sensitivity]{
\includegraphics[width=0.32\textwidth]{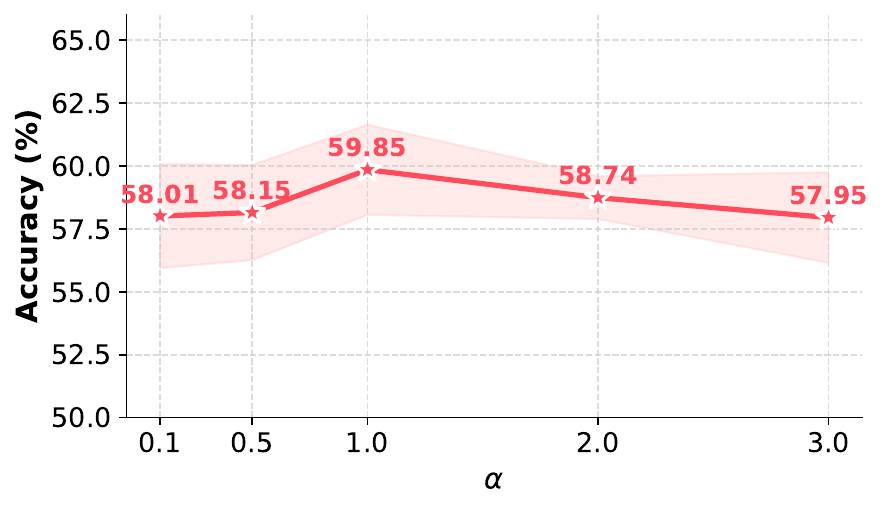}
}
\subfigure[$\mu$ sensitivity]{
\includegraphics[width=0.31\textwidth]{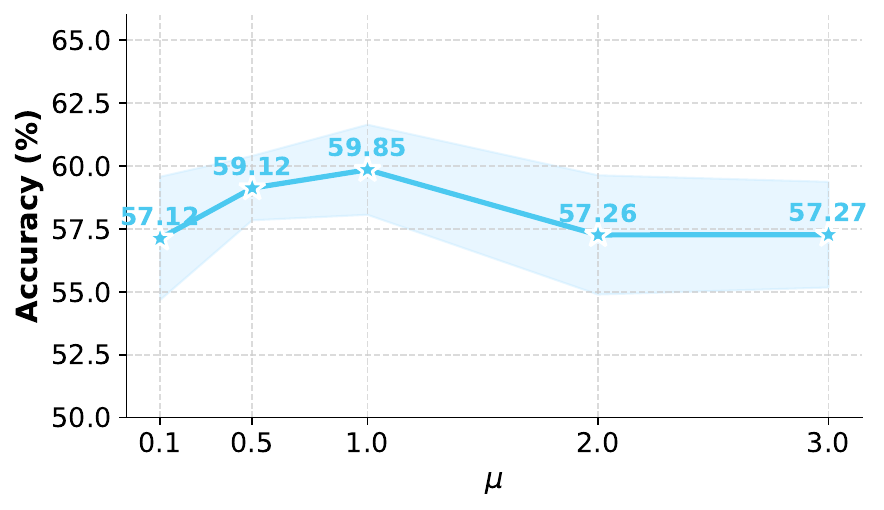}
}
\subfigure[$\nu$ sensitivity]{
\includegraphics[width=0.32\textwidth]{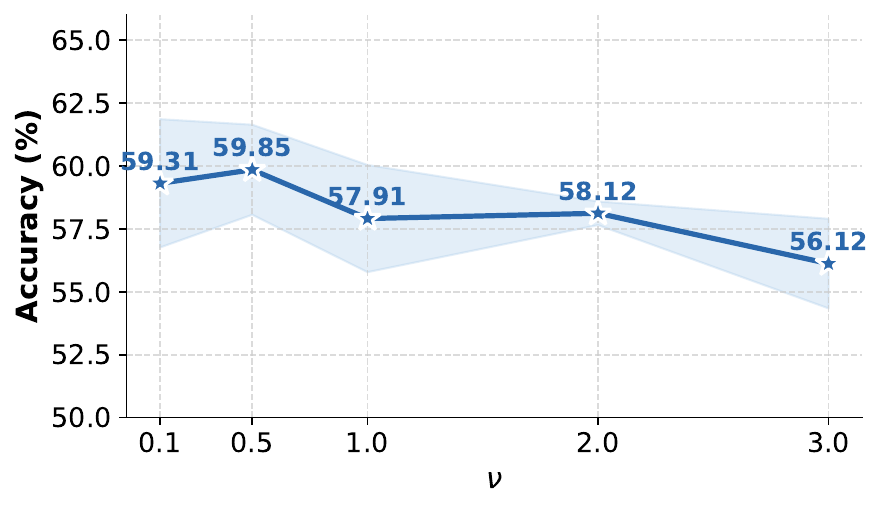}
}
\caption{Hyperparameter sensitivity analysis of \(\alpha\), \(\mu\) and \(\nu\).}
\label{fig:hyper_sen}
\end{figure}

Fig.~\ref{fig:hyper_sen}(a), Fig.~\ref{fig:hyper_sen}(b) and Fig.~\ref{fig:hyper_sen}(c) illustrate the sensitivity curves of \(\alpha\), \(\mu\) and \(\nu\), respectively. 
It is observed that our method consistently surpasses the strong baseline GraphCLIP across all parameter configurations, demonstrating the superiority and robustness of our framework. 
The model achieves optimal performance at \(\alpha=1.0\), \(\mu=1.0\) and \(\nu=0.5\). 
Either excessively small or large hyperparameters lead to performance degradation, since insufficient constraint or over-weighted loss breaks the representation learning balance. 
In addition, our method shows stable performance and small variance under different settings, which verifies that our model is insensitive to hyperparameter tuning and has good practicality.

Similarly, the reliability coefficients \(\lambda_d, \lambda_a, \lambda_c\) used in computing the sparse reliability score \(\rho_i\) are all set to 0.5. 
Since the three factors---structural stability \(d_i\), textual completeness \(a_i\), and structure--semantic consistency \(c_i\)---are normalized and inherently comparable, equal weighting is a natural choice. 
This simple setting maintains interpretability and avoids introducing additional hyperparameters, while not affecting the reliability score’s ability to reflect sample quality.

\subsection{Visualization}
\label{Visualization}
As a supplementary study of model effectiveness, we visualize both node and class representations using our model \ourmodel and the state-of-the-art baseline GraphCLIP. 
Specifically, we employ t-SNE~\cite{JMLR:v9:vandermaaten08a} to map the Cora representations into two-dimensional vectors for visualization. 
Figure~\ref{fig:cora_embeddings} shows that, compared to GraphCLIP, equipping our method \ourmodel results in: 
(1) nodes of the same class (i.e., colors in the visualization) forming more cohesive clusters in the embedding space, and 
(2) node representations from different classes being more discriminative.

\begin{figure}[H]
    \centering
    \includegraphics[width=0.98\linewidth]{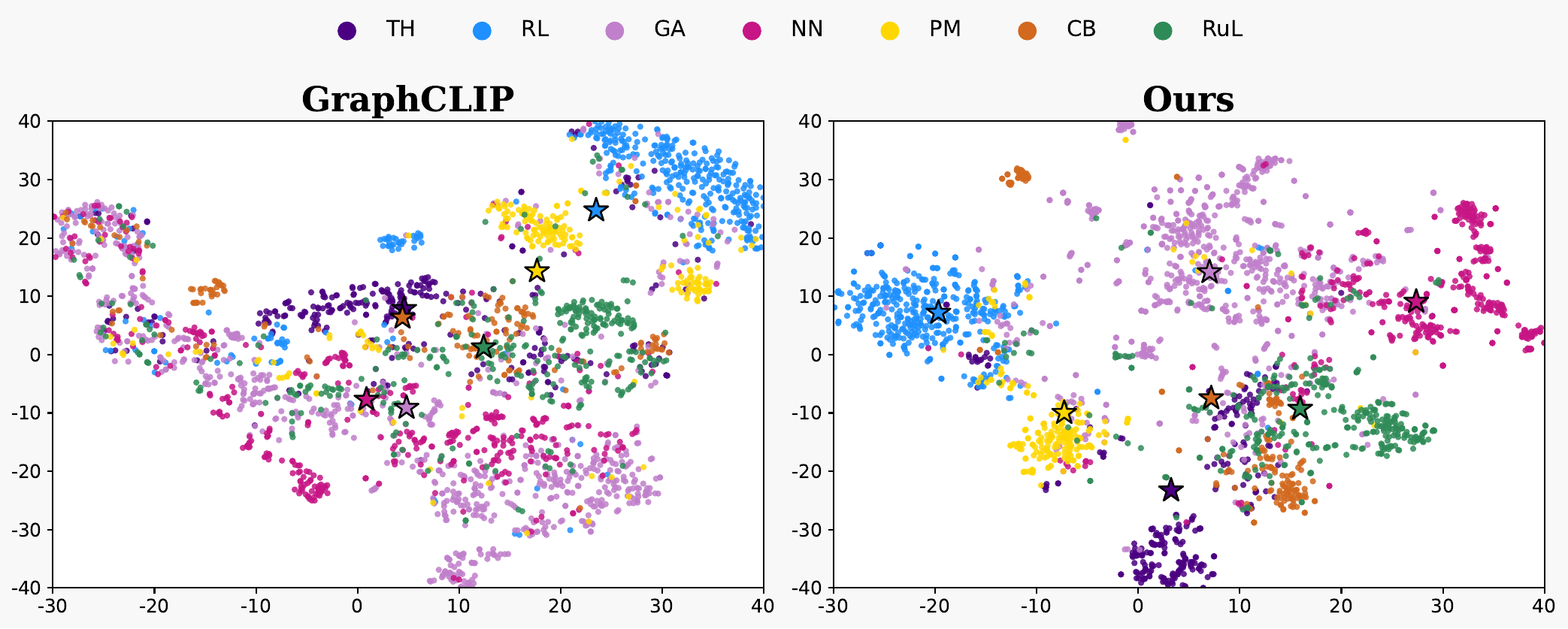} 
    \caption{Embedding visualization of Cora. Circles ($\bullet$) represent nodes, while stars ($\star$) represent classes.}
    \label{fig:cora_embeddings}
\end{figure}

\section{Complexity Analysis}
\label{Complexity}
In this section, we present the time and space complexity analysis of \ourmodel. Considering that \ourmodel is a self-supervised graph-text pre-training framework, we compare its complexity with other self-supervised graph learning methods.

\subsection{Time Complexity}
The primary time overhead arises from four components: the graph encoder, the text encoder, the structure-oriented reconstruction module, the Sparse-aware Cross-domain Risk Balancing mechanism, and computing the pre-training loss. For simplicity, we assume the layer count and hidden size of the text encoder are the same as those of the graph encoder.

The time complexity of the graph encoder is \(O(LN^2D + LND^2)\), and similarly, the time complexity of the text encoder is \(O(LN^2D + LND^2)\). The structure-oriented reconstruction and the alignment loss both operate on individual node-text pairs, leading to a time complexity of \(O(ND)\). The Sparse-aware Cross-domain Risk Balancing mechanism, including density estimation, reliability weighting, and cross-domain risk balancing, involves only lightweight feed-forward networks and sample-wise operations, resulting in a time complexity of \(O(ND)\). The pre-training loss computation has a time complexity of \(O(N^2D)\).

Thus, the total time complexity of \ourmodel is \(O(LN^2D + LND^2)\), ignoring smaller terms, which is of the same order as GraphCLIP \cite{zhu2025graphclip}, \(O(LN^2D)\).

\subsection{Space Complexity}
Each layer of the graph encoder has a space complexity of \(O(ND + D^2)\) for computing queries, keys, and values. The attention score calculation then incurs a space complexity of \(O(N^2 + ND)\), while obtaining the hidden states results in a per-layer space complexity of \(O(N^2 + ND + D^2)\). Performing these operations across all layers leads to a cumulative space complexity of \(O(LN^2 + LND + LD^2)\). Similarly, the text encoder has a space complexity of \(O(LN^2 + LND + LD^2)\). The reconstruction head, the density networks, and the reliability scoring module add negligible additional space complexity of \(O(ND)\). Finally, the contrastive and risk-balancing losses add an additional space complexity of \(O(N^2)\).

Consequently, the overall space complexity of \ourmodel amounts to \(O(LN^2 + LND + LD^2)\), which is of the same order as GraphCLIP~\cite{zhu2025graphclip}, \(O(LN^2 + LND + LD^2)\).
\newpage

\end{document}